\documentclass[letterpaper,10pt]{ieeeconf}
\IEEEoverridecommandlockouts
\usepackage{amsmath,amssymb,mathrsfs,theorem}
\usepackage{graphicx,subfigure,url}
\usepackage{epsfig,color,psfrag}
\usepackage{xspace}
\usepackage{circle}

\usepackage[vlined,ruled,linesnumbered]{algorithm2e-new}
\SetKwInput{KwOutput}{Output}
\SetKwInOut{KwInput}{\hskip-1em Input:}%
\SetKwInOut{KwOutput}{\hskip-1em Output:}%
\renewcommand{\nlSty}[1]{\textnormal{\textbf{#1}:}}

\newtheorem{theorem}{Theorem}[section]
\newtheorem{proposition}[theorem]{Proposition}

\newtheorem{definition}[theorem]{Definition}
\newtheorem{lemma}[theorem]{Lemma}

{\theorembodyfont{\rmfamily} 

\newtheorem{remarks}[theorem]{Remarks}
\newtheorem{problem}[theorem]{Problem Statement}
}

\def\QED{\mbox{\rule[0pt]{1.5ex}{1.5ex}}}

\newcommand{\PA}{\mathcal{P}}

\newcommand{\LL}{\mathcal{L}}

\newcommand{\BB}{\mathcal{B}}

\newcommand{\RR}{\mathcal{R}}
\newcommand{\MM}{\mathcal{T}}
\newcommand{\TT}{\mathcal{T}}

\newcommand{\nat}{\mathbb{N}}
\newcommand{\CC}{\mathcal{C}}

\newcommand{\real}{{\mathbb{R}}}

\newcommand{\Next}{\mathbf{X}}
\newcommand{\Always}{\mathbf{G}}
\newcommand{\Event}{\mathbf{F}}
\newcommand{\Until}{\mathcal{U}}

\newcommand{\SBP}{\textsc{Shortest-Bot-Path}}
\newcommand{\SP}{\textsc{Shortest-Path}}

\newcommand{\prop}{\alpha}
\newcommand{\opt}{\pi}

\newcommand\oprocendsymbol{\hbox{$\square$}}
\newcommand\oprocend{\relax\ifmmode\else\unskip\hfill\fi\oprocendsymbol}

\graphicspath{{./fig/}}

\urldef{\smith}\url{slsmith@mit.edu}
\urldef{\rus}\url{rus@csail.mit.edu}
\urldef{\belta}\url{cbelta@bu.edu}
\urldef{\tumovabu}\url{tumova@bu.edu}
\urldef{\tumovacz}\url{xtumova@fi.muni.cz}

\title{Optimal Path Planning under Temporal Logic Constraints}

\author{Stephen L. Smith \quad Jana T\r{u}mov\'{a} \quad Calin Belta
  \quad Daniela Rus\thanks{This material is based upon work supported
    in part by ONR-MURI Award N00014-09-1-1051 and ARO Award W911NF-09-1-0088.}
  \thanks{S. L. Smith and D. Rus are with the Computer Science and
    Artificial Intelligence Laboratory, Massachusetts Institute of
    Technology, Cambridge, MA 02139 (\smith; \rus).  J.
    T\r{u}mov\'{a} and C. Belta are with the Department of Mechanical
    Engineering, Boston University, Boston, MA 02215
    (\tumovabu;\belta).  J.  T\r{u}mov\'{a} is also affiliated with
    Faculty of Informatics, Masaryk University, Brno, Czech
    Republic.}}

\begin{document}
\maketitle
 \begin{abstract}

   In this paper we present a method for automatically generating
   optimal robot trajectories satisfying high level mission
   specifications.  The motion of the robot in the environment is
   modeled as a general transition system, enhanced with weighted
   transitions. The mission is specified by a general linear temporal
   logic formula.  In addition, we require that an \emph{optimizing}
   proposition must be repeatedly satisfied.  The cost function that
   we seek to minimize is the maximum time between satisfying
   instances of the optimizing proposition.  For every environment
   model, and for every formula, our method computes a robot
   trajectory which minimizes the cost function.

   The problem is motivated by applications in robotic monitoring and
   data gathering.  In this setting, the optimizing proposition is
   satisfied at all locations where data can be uploaded, and the
   entire formula specifies a complex (and infinite horizon) data
   collection mission.  Our method utilizes B\"uchi automata to
   produce an automaton (which can be thought of as a graph) whose
   runs satisfy the temporal logic specification.  We then present a
   graph algorithm which computes a path corresponding to the optimal
   robot trajectory.  We also present an implementation for a robot
   performing a data gathering mission in a road network.
\end{abstract}

\section{Introduction}

The goal of this paper is to plan the optimal motion of a robot
subject to temporal logic constraints.  This is an important problem
in many applications where the robot has to perform a sequence of
operations subject to external constraints.  For example, in a
persistent data gathering task the robot is tasked to gather data at
several locations and then visit a different set of upload sites to
transmit the data. Referring to~Fig.~\ref{fig:road_network}, we would
like to enable tasks such as ``Repeatedly gather data at locations P1,
P4, and P5. Upload data at either P2 or P3 after each data-gather.
Follow the road rules, and avoid the road connecting I4 to I2.''  We
wish to determine robot motion that completes the task, and minimizes
a cost function, such as the maximum time between data uploads.

Recently there has been an increased interest in using temporal logic
to specify mission plans for
robots~\cite{Antoniotti95,Loizou04,Quottrup04,
  CB-VI-GJP:04,GEF-HKG-GJP:05,Hadas-ICRA07,Tok-Ufuk-Murray-CDC09}.
Temporal logic is appealing because it provides a formal high level
language in which to describe a complex mission.  In addition, tools
from model checking~\cite{VW86,Holzmann97,Clarke99,DiVinE} can be used
to verify the existence of a robot trajectory satisfying the
specification, and can produce a satisfying trajectory. However,
frequently there are multiple robot trajectories that satisfy a given
specification.  In this case, one would like to choose the ``optimal''
trajectory according to a cost function.  The current tools from model
checking do not provide a method for doing this.  In this paper we
consider linear temporal logic specifications, and a particular form
of cost function, and provide a method for computing optimal
trajectories.

\begin{figure}
\centering
\includegraphics[width=0.8\linewidth]{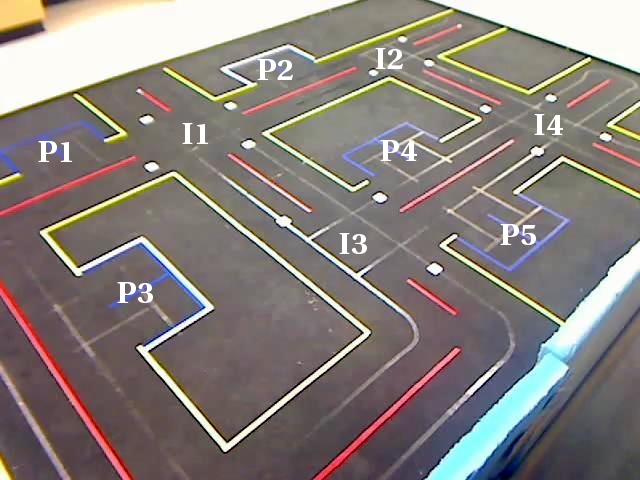}
\caption{An environment consisting of roads, intersections and parking
  lots.  An example mission in the environment is ``Repeatedly gather
  data at locations P1, P4, and P5. Upload data at either P2 or P3
  after each data-gather.  Follow the road rules, and avoid the road
  connecting I4 to I2.''}
\label{fig:road_network}
\end{figure}

The problem considered in this paper is related to the vehicle routing
problem (VRP)~\cite{PT-DV:01}.  The VRP is a generalization of the
traveling salesman problem (TSP) in which the goal is to plan routes
for vehicles to service customers.  Vehicle routing extends the TSP by
considering aspects such as multiple vehicles, vehicles with capacity
constraints, and vehicles that must depart and return to specified
depot locations.  In~\cite{SK-EF:08}, the authors consider a vehicle
routing problem with metric temporal logic constraints. The goal is to
minimize a cost function of the vehicle paths (such as total distance
traveled).  The authors present a method for computing an optimal
solution by converting the problem to a mixed integer linear program
(MILP).  However, their method only applies to specifications where
the temporal operators are applied only to atomic propositions.  Thus,
the method does not apply to persistent monitoring and data gathering
problems, which have specifications of the form ``always eventually.''
In addition, the approach that we present in this paper leads to an
optimization problem on a graph, rather than a MILP.

The contribution of this paper is to present a cost function for which
we can determine an optimal robot trajectory that satisfies a general
linear temporal logic formula.  The cost function is motivated by
problems in monitoring and data gathering, and it seeks to minimize
the time between satisfying instances of a single \emph{optimizing
  proposition}.  Our solution, summarized in the \textsc{Optimal-Run}
algorithm of Section~\ref{sec:solution}, operates as follows.  We
represent the robot and environment as a weighted transition system.
Then, we convert the linear temporal logic specification to a B\"uchi
automaton.  We synchronize the transition system with the B\"uchi
automaton creating a product automaton.  In this automaton a
satisfying run is any run which visits a set of accepting state
infinitely often.  We show that there exists an optimal run that is in
``prefix-suffix'' structure, implying that we can search for runs with
a finite transient, followed by a periodic steady-state.  Thus, we
create a polynomial time graph algorithm based on solutions of
bottleneck shortest path problems to find an optimal cycle containing
an accepting state.  We implement our solution on the physical testbed
shown in Fig.~\ref{fig:road_network}. We believe that optimizations of
this type may be useful for a broader class of problems than the one
considered here.

For simplicity of the presentation, we assume that the robot moves
among the vertices of an environment modeled as a graph. However, by
using feedback controllers for facet reachability and invariance in
polytopes \cite{HS04,HabColSchup06,Belta-TAC06} the method developed
in this paper can be easily applied for motion planning and control of
a robot with ``realistic'' continuous dynamics ({\it e.g.,} unicycle)
traversing an environment partitioned using popular partitioning
schemes such as triangulations and rectangular partitions.

The organization of the paper is as follows.  In
Section~\ref{sec:prelim} we present preliminary results in temporal
logic.  In Section~\ref{sec:problem_stat} we formally state the robot
motion planning problem, and in Section~\ref{sec:solution} we present
our solution.  In Section~\ref{sec:experiment} we present results from
a motion planning experiment for one robot in a road network
environment.  Finally in Section~\ref{sec:conc} we present some
promising future directions.

\section{Preliminaries}
\label{sec:prelim}

In this section we briefly review some aspects of linear temporal
logic (LTL).  LTL considers a finite set of variables $\Pi$, each of
which can be either true or false.  The variables $\prop_i \in \Pi$ are
called \emph{atomic propositions}. In the context of robots,
propositions can capture properties such as ``the robot is located in
region $1$'', or ``the robot is recharging.''

Given a system model, LTL allows us to express the time evolution of
the state of the system. We consider a type of finite model called the
\emph{weighted transition system}.

\begin{definition}[Weighted Transition System]
\label{def:transition system}
A weighted transition system is a tuple $\MM:=(Q,q_0,R,\Pi,\LL,w)$,
consisting of %
(i) a finite set of states $Q$; %
(ii) an initial state $q_0\in Q$; %
(iii) a transition relation $R\subseteq Q\times Q$; %
(iv) a set of atomic propositions $\Pi$; %
(v) a labeling function $\LL:Q \to 2^{\Pi}$; %
(vi) a weight function $w:R\to \real_{>0}$.%
\end{definition}
We assume that the transition system is non-blocking, implying that
there is a transition from each state.  The transition relation has
the expected definition: given that the system is in state $q_1\in Q$
at time $t_1$, the system is in state $q_2$ at time
$t_1+w\big((q_1,q_2))$ if and only if $(q_1,q_2) \in R$.  The labeling
function defines for each state $q\in Q$, the set $\LL(q)$ of all
atomic propositions valid in $q$.  For example, the proposition ``the
robot is recharging'' will be valid for all states $q\in Q$ containing
recharging stations.

For our transition system we can define a \emph{run} $r_\MM$ to be an
infinite sequence of states $q_0q_1q_2\ldots$ such that $q_0 \in Q_0$,
$q_i \in Q$, for all $i$, and $(q_i,q_{i+1}) \in R$, for all $i$.  A
run $r_\MM$ defines a \emph{word} $\LL(q_0)\LL(q_1)\LL(q_2)\ldots$
consisting of sets of atomic propositions valid at each state.  An
example of a weighted transition system is given in Fig.~\ref{fig:ts}.

\begin{figure}
\begin{center}
\scalebox{0.42}{
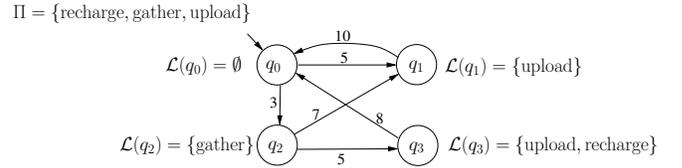
}
\caption{An example of a weighted transition system. A correct run of
  the system is for instance $q_0q_2q_1q_0q_2q_3q_0\ldots$, producing
  the word
  $\emptyset\{\text{gather}\}\{\text{upload}\}\emptyset\{\text{gather}\}\{\text{upload,recharge}\}\emptyset\ldots$.}
\label{fig:ts}
\end{center}
\end{figure}

\begin{definition}[Formula of LTL]
  An LTL formula $\phi$ over the atomic propositions $\Pi$ is defined
  inductively as follows:
$$\phi ::= \top \mid \prop \mid \phi \lor
   \phi \mid \lnot\,\phi \mid \Next\,\phi \mid \phi\,\Until\,\phi$$
 where $\top$ is a predicate true in each state of a system, $\prop \in
  \Pi$ is an atomic proposition, $\neg$ (negation) and $\vee$
  (disjunction) are standard Boolean connectives, and $\Next$ and
  $\Until$ are temporal operators.
\end{definition}

LTL formulas are interpreted over infinite runs, as those generated by
the transition system $\MM$ from Def.~\ref{def:transition system}.
Informally, $\Next\,\prop$ states that at the next state of a run,
proposition $\prop$ is true (i.e., $\prop \in \LL(q_1)$).  In
contrast, $\prop_1\,\Until\,\prop_2$ states that there is a future
moment when proposition $\prop_2$ is true, and proposition $\prop_1$
is true at least until $\prop_2$ is true.  From these temporal
operators we can construct two other useful operators Eventually
(i.e., future), $\Event$ defined as $\Event\,\phi := \top\,\Until\,
\phi$, and Always (i.e., globally), $\Always$, defined as
$\Always\,\phi := \lnot\,\Event\,\lnot\,\phi$.  The formula
$\Always\,\prop$ states that proposition $\prop$ holds at all states
of the run, and $\Event\,\prop$ states that $\prop$ holds at some
future time instance.

An LTL formula can be represented in an automata-theoretic setting as
\emph{B\"uchi automaton}, defined as follows:
\begin{definition}[B\"uchi Automaton]
  A B\"uchi automaton is a tuple $\BB :=
  (S,S_0,\Sigma,\delta,F)$, consisting of %
(i) a finite set of states $S$; %
(ii) a set of initial states $S_0\subseteq S$; %
(iii) an input alphabet $\Sigma$; %
(iv) a non-deterministic transition relation $\delta \subseteq
  S\times \Sigma \times S$; %
(v) a set of accepting (final) states $F\subseteq S$.
\end{definition}

The semantics of B\"uchi automata are defined over infinite input
words. Setting the input alphabet $\Sigma = 2^\Pi$, the semantics are
defined over the words consisting of sets of atomic propositions,
i.e. those produced by a run of the transition system. Let
$\omega=\omega_0\omega_1\omega_2\ldots$ be an infinite input word of
automaton $\BB$, where $\omega_i\in\Sigma$ for each $i\in \nat$ (for
example, the input $\omega=\LL(q_0)\LL(q_1)\LL(q_2)\ldots$ could be a word
produced by a run $q_0q_1q_2\ldots$ of the transition system $\MM$).

A \emph{run} of the B\"uchi automaton \emph{over} an input word
$\omega=\omega_0\omega_1\omega_2\ldots$ is a sequence
$r_\BB=s_0s_1s_2\ldots$, such that $s_0 \in S_0$, and
$(s_i,\omega_i,s_{i+1}) \in \delta$, for all $i\in \nat$.

\begin{definition}[B\"uchi Acceptance]
  A word $\omega$ is accepted by the B\"uchi automaton $\BB$ if and
  only if there exists $r_\BB$ over $\omega$ so that $\inf(r_\BB) \cap
  F \neq \emptyset$, where $\inf(r_\BB)$ denotes the set of states
  appearing infinitely often in run $r_\BB$.
\end{definition}

The B\"uchi automaton allows us to determine whether or not the word
produced by a run of the transition system satisfies an LTL formula.
More precisely, for any LTL formula $\phi$ over a set of atomic
propositions $\Pi$, there exists a B\"uchi automaton $\BB_{\phi}$ with
input alphabet $2^{\Pi}$ accepting all and only the infinite words
satisfying formula $\phi$~\cite{VW86}.  Translation algorithms were
proposed in~\cite{Vardi94} and efficient implementations were
developed in~\cite{Gerth95,Gastin01}. The size of the obtained B\"uchi automaton is, in general, exponential with respect to the size of the formula. However, the exponential complexity is in practice not restrictive as the LTL formulas are typically quite small.
An example of a B\"uchi automaton is given in Figure~\ref{fig:ba}.

\begin{figure}
\begin{center}
\scalebox{0.42}{
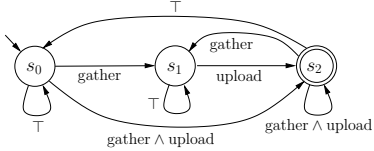
}
\caption{A B\"uchi automaton corresponding to LTL formula
  $(\Always\,\Event\,\mathrm{gather} \wedge
  \Always\,\Event\,\mathrm{upload})$ over the alphabet $\Pi$. The
  illustration of the automaton is simplified. In fact, each
  transition labeled with $\top$ represents $|2^\Pi|$ transitions
  labeled with all different subsets of atomic
  propositions. Similarly, a transition labeled with $\mathrm{gather}$
  represent $|2^\Pi|/2$ transitions labeled with all subsets of atomic
  propositions containing the proposition $\mathrm{gather}$,
  etc.}
\label{fig:ba}
\end{center}
\end{figure}

\section{Problem Statement and Approach}
\label{sec:problem_stat}

Consider a single robot in an arbitrary environment, represented as a
transition system (as defined in Section~\ref{sec:prelim})
$\mathcal{T} = (Q,q_0,R,\Pi,\LL,w)$.  A run in the transition system
starting at $q_0$ defines a corresponding trajectory of the robot in
the environment.  The time to take transition $(q_1,q_2)\in R$ (i.e.,
the time for the robot to travel from $q_1$ to $q_2$ in the
environment) is given by $w(q_1,q_2)$.

To define our problem, we assume that there is an atomic proposition
$\opt \in \Pi$, called the \emph{optimizing proposition}. We consider
LTL formulas of the form
\begin{equation}
\label{eq:temp_spec}
\phi:=\varphi \land \Always\,\Event\, \opt.
\end{equation}
The formula $\varphi$ can be any LTL formula over $\Pi$.  The second
part of the formula specifies that the proposition $\opt$ must be
satisfied infinitely often, and will simply ensure well-posedness of
our optimization.  

Let each run of $\MM$ start at time $t=0$, and assume that there is at
least one run satisfying LTL formula~\eqref{eq:temp_spec}.  For each
satisfying run $r_\MM=q_0q_1q_2\ldots$, there is a corresponding word
of sets of atomic propositions $\omega =
\omega_0\omega_1\omega_2\ldots$, where $\omega_i = \LL(q_i)$.
Associated with $r_\MM$ there is a sequence of time instances
$\mathbb{T}:=t_0,t_1,t_2,\ldots$, where $t_0 = 0$, and $t_i$ denotes
the time at which state $q_i$ is reached ($t_{i+1} = t_i +
w(q_i,q_{i+1})$).  From this time sequence we can extract all time
instances at which the proposition $\opt$ is satisfied.  We let
$\mathbb{T}_{\opt}$ denote the sequence of satisfying instances of the
proposition $\opt$.

Our goal is to synthesize an infinite run $r_\MM$ (i.e., a robot
trajectory) satisfying LTL formula~\eqref{eq:temp_spec}, and
minimizing the cost function
\begin{equation}
\label{eq:cost_fn2}
\CC(r_\MM)=\limsup_{i\to+\infty}\left(\mathbb{T}_{\opt}(i+1) -
  \mathbb{T}_{\opt}(i)\right),
\end{equation}
where $\mathbb{T}_{\opt}(i)$ is the $i$th satisfying time instance of
proposition $\opt$.  Note that a finite cost in~\eqref{eq:cost_fn2}
enforces that $\Always\,\Event\, \opt$ is satisfied. Thus, the
specification appears in $\phi$ merely to ensure that any satisfying
run has finite cost.  In summary, our goal is the following:

\begin{problem}
\label{problem:1} 
Determine an algorithm that takes as input a weighted transition
system $\MM$, an LTL formula $\phi$ in form \eqref{eq:temp_spec}, and
an optimizing proposition $\opt$, and outputs a run $r_\MM$ minimizing
the cost $\CC(r_\MM)$ in~\eqref{eq:cost_fn2}.
\end{problem}

We now make a few remarks, motivating this
problem.
\begin{remarks}[Comments on problem statement]
  \emph{Cost function form:} The transition system produces infinite
  runs.  Thus, cost function~\eqref{eq:cost_fn2} evaluates the
  steady-state time between satisfying instances of $\opt$.  In the
  upcoming sections we design an algorithm that produces runs which
  reach steady-state in \emph{finite time}.  Thus, the runs produced
  will achieve the cost in~\eqref{eq:cost_fn2} in finite time.
  
  \emph{Expressivity of LTL formula~\eqref{eq:temp_spec}:} Many
  interesting LTL specifications can be cast in the form
  of~\eqref{eq:temp_spec}.  For example, suppose that we want to
  minimize the time between satisfying instances of a disjunction of
  propositions $\vee_{i} \prop_i$.  We can write this in the
  formula~\eqref{eq:temp_spec} by defining a new proposition $\opt$ which
  is satisfied at each state in which a $\prop_i$ is satisfied.

In addition, the LTL formula $\varphi$ in (\ref{eq:temp_spec}) allows
us to specify various rich robot motion requirements. An example of
such is \emph{global absence} ($\Always \, \neg \psi$, globally keep
avoiding $\psi$), \emph{response} ($\Always \, (\psi_1 \Rightarrow
\Event \, \psi_2)$, whenever $\psi_1$ holds true, $\psi_2$ will happen
in future), \emph{reactivity} ($\Always \, \Event \, \psi_1
\Rightarrow \Always \, \Event \, \psi_2$, if $\psi_1$ holds in future
for any time point, $\psi_2$ has to happen in future for any time
point as well), \emph{sequencing} ($\psi_1 \, \Until \, \psi_2 \,
\Until \, \psi_3$, $\psi_1$ holds until $\psi_2$ happens, which holds
until $\psi_3$ happens), and many others. For concrete examples, see
Section~\ref{sec:experiment}. \oprocend
\end{remarks}

\section{Problem Solution}
\label{sec:solution}

In this section we describe our solution to Problem~\ref{problem:1}.
We leverage ideas from the automata-theoretic approach to model
checking.

\subsection{The Product Automaton}

Consider the weighted transition system $\MM=(Q,q_0,R,\Pi,\LL,w)$, and
a proposition $\opt \in\Pi$.  In addition, consider an LTL formula
$\phi=\varphi \land \Always\,\Event\,\opt $ over $\Pi$ in
form~(\ref{eq:temp_spec}), translated into a B\"uchi automaton
$\BB_\phi=(S,S_0,2^\Pi,\delta,F)$.  
With these two components, we define a new object, which we call the
\emph{product automaton}, that is suitably defined for our problem.

\begin{definition}[Product Automaton]
  The product automaton $\PA = \MM \times \BB_\phi$ between the
  transition system $\MM$ and the B\"uchi automaton $\BB_\phi$ is
  defined as the tuple $\PA :=
  (S_{\PA},S_{\PA,0},\delta_{\PA},F_{\PA},w_\PA,S_{\PA,\opt})$,
  consisting of
  \begin{enumerate}
  \item a finite set of states $S_{\PA}= Q \times S$,
  \item a set of initial states $S_{\PA,0}=\{q_0\}\times S_0$,
 \item a transition relation $\delta_{\PA} \subseteq S_{\PA}\times
    S_{\PA}$, where $\big((q,s),(\bar q,\bar s)\big)\in \delta_{\PA}$
    if and only if $(q,\bar q)\in R$ and $(s,\LL(q),\bar s) \in \delta$.
\item a set of accepting (final) states $F_{\PA} = Q\times F$.
\item a weight function $w_\PA: \delta_\PA \to \real_{>0}$, where
  $w_\PA\big((q,s),(\bar q,\bar s)\big) = w(q,\bar q)$, for all
  $\big((q,s),(\bar q,\bar s)\big)\in \delta_\PA$.
\item a set of states $S_{\PA,\opt}\subseteq S_\PA$ in which the
  proposition $\opt$ holds true. Thus, $(q,s)\in S_{\PA,\opt}$ if and
  only if $\opt \in \LL(q)$.
  \end{enumerate}
\end{definition}

The product automaton (as defined above) can be seen as a B\"uchi
automaton with a trivial input alphabet.  Since the alphabet is
trivial, we omit it.  Thus, we say that a run $r_\PA$ in product
automaton $\PA$ is accepting if $\inf(r_\PA)\cap F_\PA \neq
\emptyset$.  An example product automaton is illustrated in
Fig.~\ref{fig:product}.

\begin{figure}
\begin{center}
\scalebox{0.42}{
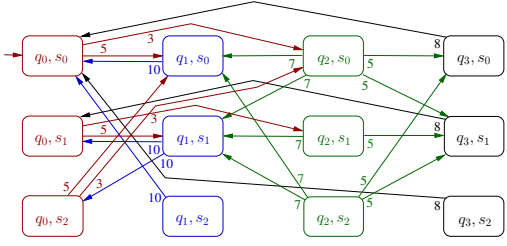
}
\caption{Product automaton between the transition system in Figure \ref{fig:ts} and the B\"uchi automaton in Figure \ref{fig:ba}.}
\label{fig:product}
\end{center}
\end{figure}

As in the transition system, we associate with each run
$r_\PA=p_0p_1p_2\ldots$, a sequence of time instances
$\mathbb{T}_\PA:=t_0t_1t_2\ldots$, where $t_0 = 0$, and $t_i$ denotes
the time at which the $i$th vertex in the run is reached ($t_{i+1} =
t_i + w_\PA(p_i,p_{i+1})$).  From this time sequence we can extract a
sequence $\mathbb{T}_{\PA,\opt}$, containing time instances $t_i$,
where $p_i \in S_{\PA,\opt}$ (i.e. $\mathbb{T}_{\PA,\opt}$ is a sequence
of satisfying instances of the optimizing proposition $\opt$ in $\MM$).
The cost of a run $r_\PA$ on the product automaton $\PA$ (which
corresponds to cost function~\eqref{eq:cost_fn2} on transition system
$\MM$) is
\begin{equation}
\label{eq:cost_fn_prod}
\CC_\PA(r_\PA)=\limsup_{i\to+\infty} \left(\mathbb{T}_{\PA,\opt}(i+1) -
  \mathbb{T}_{\PA,\opt}(i)\right).
\end{equation}

The product automaton can also be viewed as a weighted graph, where
the states define vertices of the graph and the transitions define the
edges. Thus, we at times refer to runs of the product automaton as
{\em paths}. A \emph{finite path} is then a finite fragment of an
infinite path.

Each accepting run of the product automaton can be projected to a run
of the transition system satisfying the LTL formula.  Formally, we
have the following.
\begin{proposition}[Product Run Projection,~\cite{VW86}]
\label{prop:proj}
  For any accepting run $r_{\PA}=(q_0,s_0)(q_1,s_1)(q_2,s_2)\ldots$ of
  the product automaton $\PA$, the sequence $r_\MM=q_0q_1q_2\ldots$ is
  a run of $\MM$ satisfying $\phi$.  Furthermore, the values of cost
  functions $\CC_\PA$ and $\CC$ are equal for runs $r_\PA$ and $r_\TT$,
  respectively.

  Similarly, if $r_\MM = q_0q_1q_2\ldots$ is a run of $\MM$ satisfying
  $\phi$, then there exists an accepting run
  $r_{\PA}=(q_0,s_0)(q_1,s_1)(q_2,s_2)\ldots$ of the product automaton
  $\PA$, such that the values of cost functions $\CC$ and $\CC_\PA$
  are equal.
\end{proposition}

Finally, we need to discuss the structure of an accepting run of a
product automaton $\PA$.   
\begin{definition}[Prefix-Suffix Structure]
  A \emph{prefix} of an accepting run is a finite path from an initial
  state to an accepting state $f\in F_\PA$ containing no other
  occurrence of $f$.  A \emph{periodic suffix} is an infinite run
  originating at the accepting state $f$ reached by the prefix, and
  periodically repeating a finite path originating and ending at $f$,
  and containing no other occurrence of $f$ (but possibly containing
  other vertices in $F_\PA$).  An accepting run is in
  \emph{prefix-suffix} structure if it consists of a prefix followed
  by a periodic suffix.
\end{definition}
Intuitively, the prefix can be thought of as the transient, while the
suffix is the steady-state periodic behavior.

\begin{lemma}[Prefix-Suffix Structure]
\label{lemma:1}
At least one of the accepting runs $r_\PA$ of $\PA$ that minimizes
cost function $\CC_\PA(r_\PA)$ is in prefix-suffix structure.
\end{lemma}
\begin{proof}
  Let $r_\PA$ be an accepting run that minimizes cost function
  $\CC_\PA(r_\PA)$ and is not in prefix-suffix structure. We will
  prove the existence of an accepting run $\rho_\PA$ in prefix-suffix
  structure, such that $\CC_\PA(\rho_\PA) \leq \CC_\PA(r_\PA)$.  The
  idea behind the proof is that an accepting state must occur
  infinitely many times on $r_\PA$.  We then show that we can extract
  a finite path starting and ending at this accepting state which can
  be repeated to form a periodic suffix whose cost is no larger than
  $\CC_\PA(r_\PA)$.

  To begin, there exists a state $f\in F_\PA$ occurring on $r_\PA$
  infinitely many times. Run $r_\PA$ consists of a prefix
  $r_\PA^{\mathrm{fin}}$ ending at state $f$ followed by an infinite,
  non-periodic suffix $r_\PA^{\mathrm{suf}}$ originating at the state
  $f$ reached by the prefix. The suffix $r_\PA^\mathrm{suf}$ can be
  viewed as infinite number of finite paths of form $fp_1p_2\ldots
  p_nf$, where $p_i \neq f$ for any $i\in \{1,\ldots,n\}$. Let $\RR$
  denote the set of all finite paths of the mentioned form occurring
  on the suffix $r_\PA^{\mathrm{suf}}$.

  Note, that each path in the set $\RR$ has to contain at least one
  occurrence of a state from $S_{\PA,\opt}$. To see this, assume by
  way of contradiction that there is a path $fp_1p_2\ldots p_nf$ that
  does not contain any state from $S_{P,\opt}$. The prefix
  $r_\PA^{\mathrm{fin}}$ followed by infinitely many repetitions of
  this path is indeed an accepting run of $\PA$. However, if projected
  into run of $\TT$, formula $\Always\, \Event\, \opt$ and thus also
  formula $\phi$ is violated, contradicting Proposition~\ref{prop:proj}.

  Similarly as for infinite paths, we associate with each finite path
  of length $n$ a sequence of time instances
  $\mathbb{T}_\PA:=t_0t_1t_2\ldots t_n$, where $t_0 = 0$, and $t_i$
  denotes the time at which the $i$th vertex in the run is reached
  ($t_{i+1} = t_i + w_\PA(p_i,p_{i+1})$).  From this time sequence we
  can extract a sequence $\mathbb{T}_{\PA,\opt}$, containing time
  instances $t_i$, where $p_i \in S_{\PA,\opt}$.

  For each finite path $r\in \RR$ with $n$ states and $k$ occurrences
  of a state from $S_{\PA,\opt}$ we define the following three costs
\begin{itemize}
\item $c^{f\leadsto}(r) = \mathbb{T}_{\PA,\opt}(0) -
  \mathbb{T}_{\PA}(0) $ 
\item $c(r) = \max_{i\in   \{0,\ldots,k-1\}}\left( \mathbb{T}_{\PA,\opt}(i+1) - \mathbb{T}_{\PA,\opt}(i)\right)$
\item $c^{\,\leadsto f}(r) = \mathbb{T}_{\PA}(n) - \mathbb{T}_{\PA,\opt}(k)$.
\end{itemize}

Further, we define an equivalence relation $\sim$ over $\RR$ as follows. Let $r_1,r_2\in \RR$. $r_1 \sim r_2$ if and only if
\begin{itemize}
 \item $c^{f\leadsto}(r_1)=c^{f\leadsto}(r_2)$, 
 \item $c(r_1)=c(r_2)$, and
 \item $c^{\,\leadsto f}(r_1)=c^{\,\leadsto f}(r_2)$.
\end{itemize}
Costs $c^{f\leadsto}$, $c$, and $c^{\leadsto f}$ can be extended to
$c^{f\leadsto}_{\sim}$, $c_\sim$, and $c^{\leadsto f}_{\sim}$ in a
natural way. For example, we define $c^{f\leadsto}_{\sim}([r]_\sim) =
c^{f\leadsto}(r)$, where $r \in [r]_\sim$. The other two costs are
defined analogously.

Let us extract a set $\RR^{\mathrm{inf}}/_\sim$ from the set of
equivalence classes $\RR/_\sim$ such that each class in
$\RR^{\mathrm{inf}}/_\sim$ is infinite or contains a finite path that
is repeated in $r_\PA$ infinitely many times. As a consequence, for
each class $[r]_\sim$ in $\RR^{\mathrm{inf}}/_\sim$, it holds that
$c_\sim([r]_\sim) \leq \CC_\PA(r_\PA)$. The set $\RR/_\sim$ is finite,
because there is only a finite number of different values of
costs. Furthermore, accepting run $r_\PA$ is infinite and thus
$\RR^{\mathrm{inf}}/_\sim$ is nonempty.

Let $[\rho]_\sim \in \RR^{\mathrm{inf}}/_\sim $ now be a class such
that $c^{f\leadsto}_{\sim}([\rho]_\sim)$ is minimal among the classes
from $\RR^{\mathrm{inf}}/_\sim$. 

Each time a finite path in $[\rho]_\sim$ appears in $r_{\PA}$, it
is followed by another finite path.  Consider, that infinitely many times the
``following'' path comes from a class $([r]_\sim) \in
\RR^{\mathrm{inf}}/_\sim$. Then, we must have $c^{\leadsto
  f}([\rho]_\sim) + c^{f\leadsto}([r]_\sim)\leq \CC_\PA(r_\PA)$.
But, $c^{f\leadsto}([r]_\sim) \geq c^{\leadsto f}([\rho]_\sim)$,
and thus $c^{\leadsto f}([\rho]_\sim) + c^{f
  \leadsto}([\rho]_\sim)\leq \CC_\PA(r_\PA)$.

Thus we can build the run $\rho_\PA$ as the prefix
$r_\PA^{\mathrm{fin}}$ followed by a periodic suffix
$\rho_\PA^\mathrm{suf}$, which is obtained by infinitely many
repetitions of an arbitrary path $\rho \in [\rho]_\sim$. $\rho_\PA$ is
in prefix-suffix structure and for its suffix $\rho_\PA^\mathrm{suf}$
it also holds $\CC_\PA(\rho_\PA)=\max_{i \in \nat}
\big(\mathbb{T}_{\PA,\opt}(i+1)-\mathbb{T}_{\PA,\opt}(i+1)\big)=\max\big(c(\rho),c^{f\leadsto}(\rho)+c^{\leadsto
  f}(\rho) \big) \leq \CC_\PA(r_\PA)$.
\end{proof}

\begin{definition}[Suffix Cost]
  The cost of the suffix $p_0p_1\ldots p_np_0p_1\ldots$ of a run
  $r_\PA$ is defined as follows. Let $t_{0,0},t_{0,1},\ldots, t_{0,n},
  t_{1,0}, t_{1,1} \ldots$ be the sequence of times at which the
  vertices of the suffix are reached on run $r_\PA$.  Extract the
  sub-sequence $\mathbb{T}_\PA^{\mathrm{suf}}$ of times $t_{i,j}$,
  where $p_j \in S_{\PA,\opt}$ (i.e. the satisfying instances of
  proposition $\opt$ in transition system $\MM$).  Then, the cost of
  the suffix is
\[
\CC_\PA^{\mathrm{suf}}(r_\PA) = \max_{i\in \nat}(\mathbb{T}_\PA^{\mathrm{suf}}(i+1) -
  \mathbb{T}_\PA^{\mathrm{suf}}(i)). 
\]
\end{definition}
 
From the definition of the product automaton cost $\CC_\PA$ and the
suffix cost $\CC_\PA^{\mathrm{suf}}$ we obtain the following result.
\begin{lemma}[Cost of a Run]
\label{lemma:2}
Given a run $r_\PA$ with prefix-suffix structure and its suffix
$p_0p_1p_2\ldots p_np_0p_1\ldots$, the value of the cost function $\CC_\PA
(r_\PA)$ is equal to the cost of the suffix
$\CC_\PA^{\mathrm{suf}}(r_\PA)$.
\end{lemma}

Our aim is to synthesize a run $r_\MM$ of $\MM$ minimizing the cost function
$\CC(r_\MM)$ and ensuring that the word produced by this run will be accepted
by $\BB$. This goal now translates to generating a run $r_\PA$ of $\PA$, such
that the run satisfies the B\"uchi condition $F_\PA$ and minimizes
cost function $\CC_\PA(r_\PA)$. Furthermore, to find a satisfying run $r_\PA$ that
minimizes $\CC_\PA(r_\PA)$, it is enough to consider runs in prefix-suffix
structure (see Lemma~\ref{lemma:1}). From Lemma~\ref{lemma:2} it
follows that the whole problem reduces to finding a periodic suffix
$r_\PA^{\mathrm{suf}} = fp_1p_2\ldots p_nfp_1\ldots$ in $\PA$, such that:
\begin{enumerate}
\item $f$ is reachable from an initial state in $S_{\PA,0}$,
 \item $f \in F_\PA$ (i.e., $f$ is an accepting state), and
 \item the cost of the suffix $r_\PA^{\mathrm{suf}}$ is minimum among
   all the suffices satisfying (i) and (ii).  
\end{enumerate}
Finally, we can find a finite prefix in $\PA$ leading from an initial
state in $S_{\PA,0}$ to the state $f$ in the suffix
$r_\PA^{\mathrm{suf}}$. By concatenating the prefix and suffix, we
obtain an optimal run in $\PA$.  By projecting the optimal run to
$\MM$, via Proposition~\ref{prop:proj}, we obtain a solution to our
stated problem.

\subsection{Graph Algorithm for Shortest Bottleneck Cycles}

We now focus on finding an optimal suffix in the product automaton.
We cast this problem as path optimization on a graph.  To do this, let
us define some terminology.

A graph $G=(V,E,w)$ consists of a vertex set $V$, an edge set
$E\subseteq V\times V$, and a weight function $w: E\to \real_{>0}$.  A
\emph{cycle} in $G$ is a vertex sequence $v_1v_2\ldots v_kv_{k+1}$,
such that $(v_i,v_{i+1})\in E$ for each $i\in\{1,\ldots,k\}$, and $v_1
=
v_{k+1}$. 
Given a vertex set $S\subseteq V$, consider a cycle
$c=v_1\ldots v_kv_{k+1}$ containing at least one vertex in $S$.  Let
$(i_1,i_2,\ldots,i_s)$ be the ordered set of vertices in $c$ that are
elements of $S$ (i.e., Indices with order $i_1 < i_2 < \cdots < i_m$,
such that $v_j \in S$ if and only if $j\in \{i_1,i_2,\ldots, i_s\}$).
Then, the \emph{$S$-bottleneck length} is
\[
\max_{\ell \in \{1,\ldots,s\}}\sum_{j=i_{\ell}}^{i_{\ell+1}-1}w(e_j),
\]
where $i_{s+1} = i_1$.  In words, we $S$-bottleneck distance is
defined as follows.
\begin{definition}[$S$-bottleneck length]
  Given a graph $G=(V,E,w)$, and a vertex set $S\subseteq V$, the
  $S$-bottleneck length of a cycle in $G$ is the maximum
  distance between successive appearances of an element of $S$ on the
  cycle.\footnote{If the cycle does not contain an element of $S$, then its
  $S$-bottleneck length is defined as $+\infty$.}
\end{definition}
The \emph{bottleneck length} of a cycle is defined as the maximum
length edge on the cycle~\cite{BK-JV:07}.  In contrast, the
$S$-bottleneck length measures distances between vertices in $S$.

With the terminology in place, our goal is to solve the
\emph{constrained $S$-bottleneck problem}:
\begin{problem}
  \label{problem:2}
  Given a graph $G = (V,E,w)$, and two vertex sets $F,S\subseteq V$,
  find a cycle in $G$ containing at least one vertex in $F$, with
  minimum $S$-bottleneck length.
\end{problem}

Our solution, shown in the \textsc{Min-Bottleneck-Cycle} algorithm,
utilizes Dijkstra's algorithm~\cite{BK-JV:07} for computing shortest
paths between pairs of vertices (called \textsc{Shortest-Path}), and a
slight variation of Dijkstra's algorithm for computing shortest
bottleneck paths between pairs of vertices (called
\textsc{Shortest-Bot-Path}).

\textsc{Shortest-Path} takes as inputs a graph $G = (V,E,w)$, a set of
source vertices $A\subseteq V$, and a set of destination vertices
$B\subseteq V$.  It outputs a distance matrix $D\in
\real^{|A|\times|B|}$, where the entry $D(i,j)$ gives the
shortest-path distance from $A_i$ to $B_j$.  It also outputs a
predecessor matrix $P\in V^{|A|\times |V|}$, where $P(i,j)$ is the
predecessor of $j$ on a shortest path from $A_i$ to $V_j$.  For a
vertex $v\in V$, the shortest path from $v$ to $v$ is defined as the
shortest cycle containing $v$.  If there does not exist a path between
vertices, then the distance is $+\infty$.

\textsc{Shortest-Bot-Path} has the same inputs as
\textsc{Shortest-Path}, but it outputs paths which minimize the
maximum edge length, rather than the sum of edge lengths.

Fig.~\ref{fig:directed_example} (left) shows an example input to the
algorithm.  The graph contains 12 vertices, with one vertex (diamond)
in $F$, and four vertices (square) in $S$.
Fig.~\ref{fig:directed_example} (right) shows the optimal solution as
produced by the algorithm.  The bottleneck occurs between the square
vertices immediately before and after the diamond vertex.

\begin{algorithm} 
  \dontprintsemicolon %
 \KwInput{A directed graph $G$, and vertex subsets $F$ and $S$} %
  \KwOutput{A cycle in $G$ which contains at least one vertex in $F$
    and minimizes the $S$-bottleneck distance.}  
Shortest paths between vertices in $S$:
\[
(D,P) \leftarrow \SP(G,S,S). 
\]
\; %
Define a graph $G_S$ with vertices $S$ and adjacency matrix~$D$. \;%
Shortest $S$-bottleneck paths between vertices in $S$:
\[
(D_{\mathrm{bot}},P_{\mathrm{bot}})\leftarrow\SBP(G_S,S,S).
\]
\; %
Shortest paths from each vertex in $F$ to each vertex in $S$, and from
each vertex in $S$ to each vertex in $F$:
\begin{align*}
(D_{F\to S},P_{F\to S})&\leftarrow \SP(G,F,S) \\
(D_{S\to F},P_{S\to F})&\leftarrow \SP(G,S,F).
\end{align*}
Set $D_{F\to S}(i,j) = 0$ and $D_{S\to F}(j,i) = 0$ for all $i,j$ such
that $F_i = S_j$.
\;%
For each triple $(f,s_1,s_2) \in F\times S\times S$, let
$C(f,s_1,s_2)$ be $D_{F \to S}(f,s_1) + D_{S \to F}(s_2,f)$, if $f\neq
s_1 = s_2$, and $\max\big\{D_{F \to S}(f,s_1) + D_{S \to F}(s_2,f),
D_{\mathrm{bot}}(s_1,s_2)\big\}$, otherwise.\;%
Find the triple $(f^*,s_1^*,s_2^*)$ that minimizes $C(f,s_1,s_2)$.
\;%
If minimum cost is $+\infty$, then output ``no cycle exists.''  Else,
output cycle by extracting the path from $f^*$ to $s_1^*$ using
$P_{F\to S}$, the path from $s_1^*$ to $s_2^*$ using
$P_{\mathrm{bot}}$ and $P$, and the path from $s_2^*$ to $f^*$ using
$P_{S\to F}$. \; %
\caption{\textsc{Min-Bottleneck-Cycle}$(G,S,F)$}
\label{alg:main_alg}
\end{algorithm}

\begin{figure}
\centering
\includegraphics[width=0.48\linewidth]{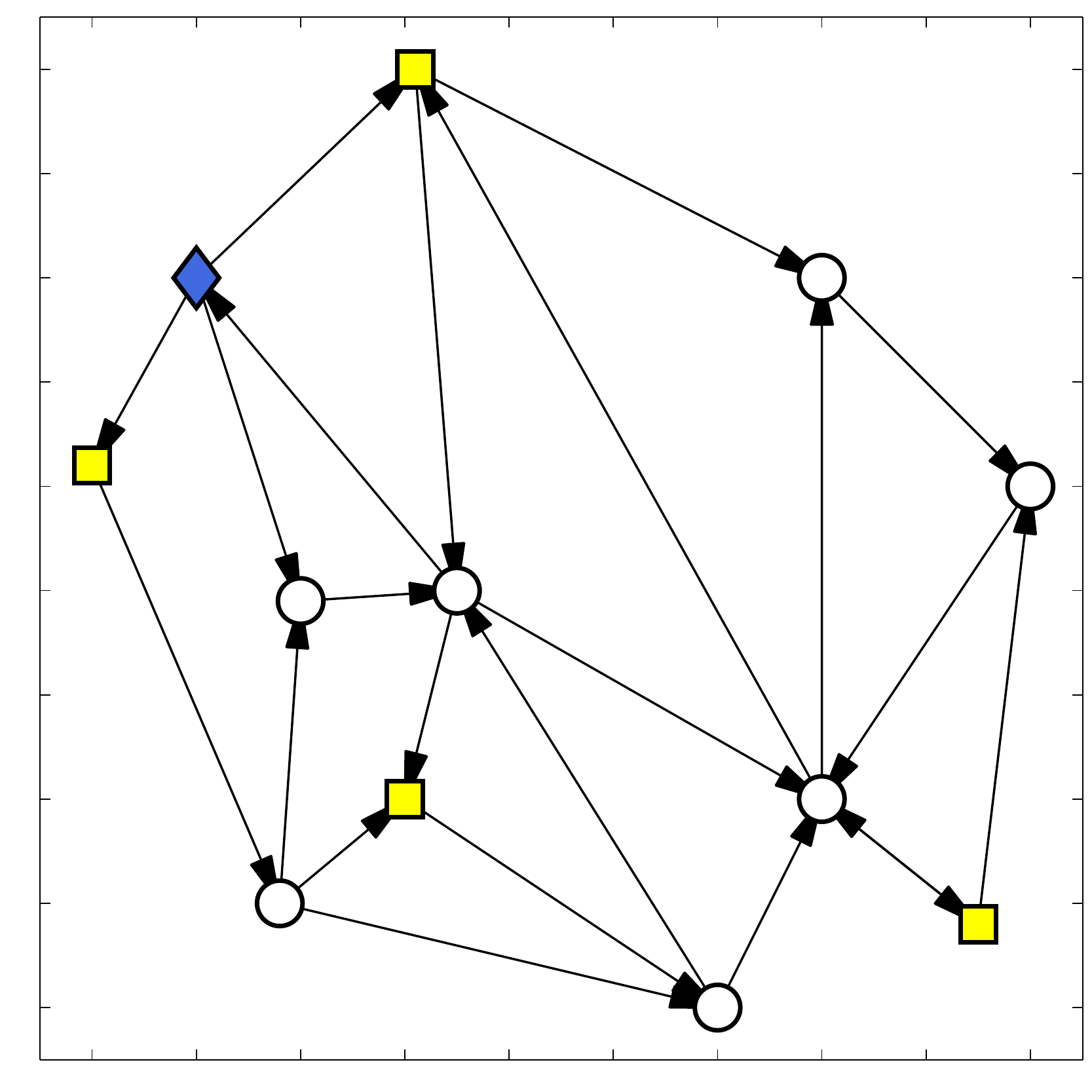}
\hfill
\includegraphics[width=0.48\linewidth]{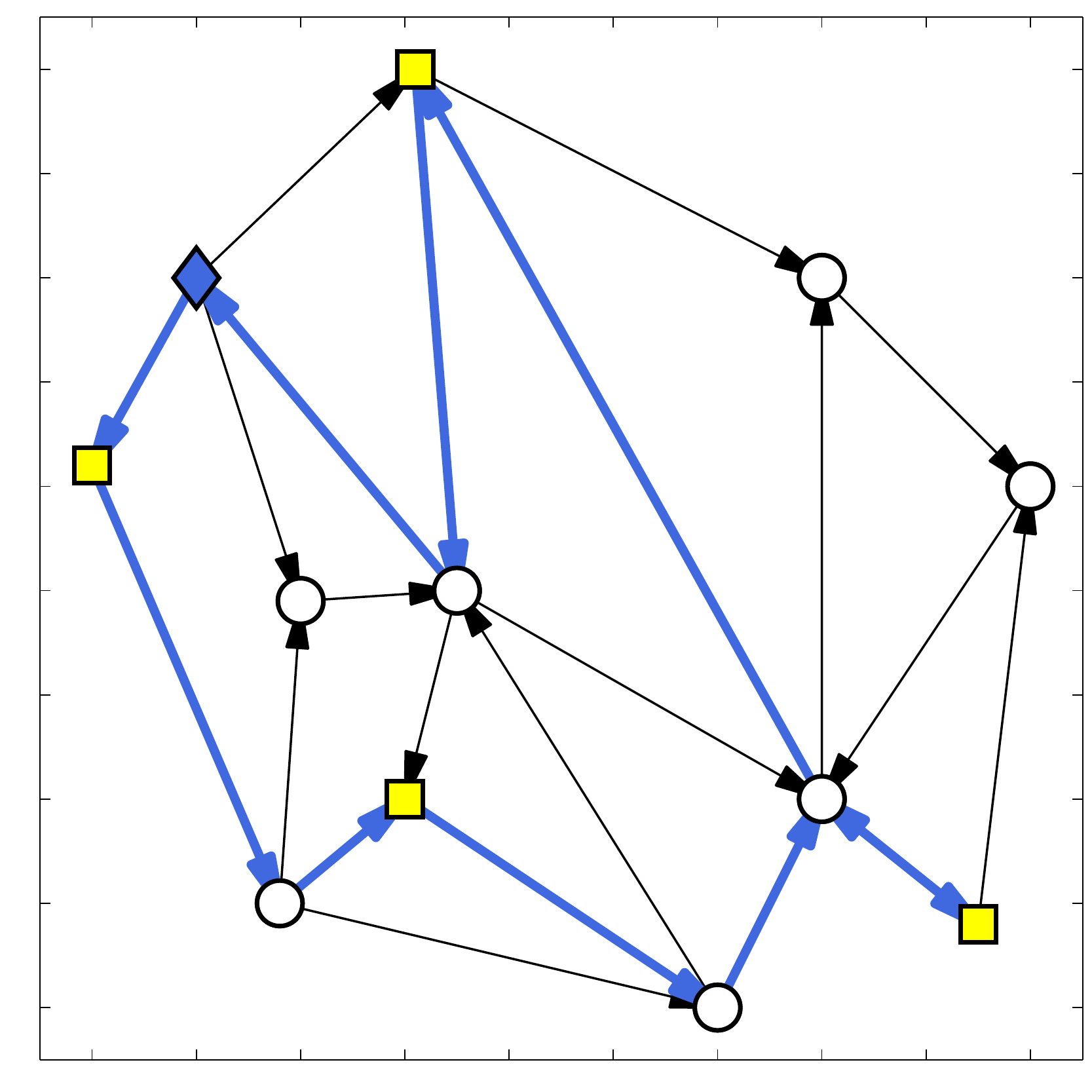}
\caption{A directed graph for illustrating the algorithm.  The edge
  weights are given by the Euclidean distance.  The set $F$ is a
  singleton given by the blue diamond.  The vertices in $S$ are drawn
  as yellow squares. The thick blue edges in the right figure form a
  cycle with minimum $S$-bottleneck length.}
\label{fig:directed_example}
\end{figure}
In the algorithm, one has to take special care that cycle lengths are
computed properly when $f=s_1$, $s_1=s_2$, or $f=s_2$.  This is done
by setting some entries of $D_{F\to S}$ and $D_{S\to F}$ to zero in
step~4, and by defining the cost differently when $f\neq s_1=s_2$ in
step~5.  In the following theorem we show the correctness of the
algorithm.

\begin{theorem}[\textsc{Min-Bottleneck-Cycle} Optimality]
\label{thm:min_bot_cyc}
The \textsc{Min-Bottleneck-Cycle} algorithm solves the constrained
$S$-bottleneck problem (Problem~\ref{problem:2}).
\end{theorem}
\begin{proof}
  Every valid cycle must contain at least one element from $F$ and at
  least one element from $S$.  Let $c:=v_1v_2\ldots v_kv_1$, be a
  valid cycle, and without loss of generality let $v_1\in F$.  From
  this cycle we can extract the triple $(v_1,v_a,v_b)\in F\times
  S\times S$, where$v_a,v_b\in S$, and $v_i \notin S$ for all $i <a$
  and for all $i > b$. (Note that, $a=b=1$ is possible.)

  Consider a cycle $c$ with corresponding triple $(f,s_1,s_2)$, and
  let $L(c)$ denote its $S$-bottleneck length.  It is straightforward
  to verify, using the definition of $S$-bottleneck length, that $L(c)
  \geq C(f,s_1,s_2)$. 

 The cycle computed in step 5 (as given by the four predecessor
  matrices) takes the shortest path from $f$ to $s_1$, the shortest
  $S$-bottleneck path from $s_1$ to $s_2$, and the shortest path from
  $s_2$ to $f$.  However, the shortest path from $f$ to $s_1$ (and
  from $s_2$ to $f$) may contain other vertices from $S$.  Thus, the
  $S$-bottleneck length of this cycle, denoted $L(f,s_1,s_2)$,
  satisfies
  \begin{equation}
    \label{eq:length_bd}
    L(f,s_1,s_2) \leq C(f,s_1,s_2) \leq L(c),
  \end{equation}
  implying that $C(f,s_1,s_2)$ upper bounds the length of the computed
  cycle.  However, if we take $c$ to be a cycle with minimum length,
  then necessarily $L(c) \leq L(f,s_1,s_2)$.  Hence,
  equation~\eqref{eq:length_bd} implies that for an optimal cycle,
  $L(f,s_1,s_2) = C(f,s_1,s_2) = L(c)$.  Thus, by minimizing the cost
  function in step 5 we compute the minimum length cycle.
\end{proof}

\textbf{Computational Complexity:} Finally, we characterize the
computational complexity of the \textsc{Min-Bottleneck-Cycle}
algorithm.  Let $n$, $m$, $n_S$, and $n_F$, be the number of vertices
(edges) in the sets $V$, $E$, $S$, and $F$, respectively.  Dijkstra's
algorithm can be implemented to compute shortest paths from a source
vertex $v\in V$, to all other vertices in $V$ in $O(n\log n+m)$ run
time.  Thus, for sparse graphs (which includes many transition
systems), the run time is $O(n\log n)$.

\begin{proposition}[\textsc{Min-Bottleneck-Cycle} run time]
  The run time of the \textsc{Min-Bottleneck-Cycle} algorithm is
  $O\big( (n_S+n_F)(n\log n + m + n_S^2) \big)$.  Thus, in the
  worst-case, the run time is $O(n^3)$.  For sparse graphs with
  $n_S,n_F\ll n$, the run time is $O\big( (n_S+n_F) n\log n\big)$.
\end{proposition}
\begin{proof}
  We simply look at the run time of each step in the algorithm.  Step
  1 requires $n_S$ calls to Dijkstra's algorithm, and has run time
  $O(n_S(n\log(n) + m))$. Step 3 requires $n_S$ calls to Dijkstra's
  algorithm on a smaller graph $G_S= (S,E_S,w_S)$, and has run time
  $O(n_S(n_S\log(n_S) + |E_S|))$.  Step 4 has run time $O(n_F(n\log(n)
  + m))$.  Finally, step 5 and 6 require searching over all $n_F\cdot
  n_S^2$ possibilities, and have run time $O(n_Fn_S^2)$.  Since $|E_S|
  \leq n_S^2$, the run time in general is given by
  $O\big((n_S+n_F)(n\log n + m + n_S^2) \big)$.
\end{proof}

\subsection{The \textsc{Optimal-Run} algorithm}

We are now ready to combine the results from the previous section to
present a solution to Problem~\ref{problem:1}.  The solution is
summarized in the \textsc{Optimal-Run} algorithm.
\begin{algorithm} 
  \dontprintsemicolon %
 \KwInput{A weighted transition system $\MM$, and temporal logic
    specification $\phi$ in form~\eqref{eq:temp_spec}.} %
  \KwOutput{A run in $\MM$ which satisfies $\phi$ and
    minimizes~\eqref{eq:cost_fn2}.}%
  Convert $\phi$ to a B\"uchi automaton $\BB_{\phi}$. \;%
  Compute the product automaton $\PA = \MM\times \BB_{\phi}$. \;
  Compute the cycle
  \textsc{Min-Bottleneck-Cycle}$(G,S_{\PA,\pi},F_\PA)$, where $G =
  (S_\PA,\delta_\PA,w_\PA)$. \; %
  Compute a shortest path from $S_{\PA, 0}$ to the cycle. \;%
  Project the complete run (path and cycle) to a run on $\MM$ using
  Proposition~\ref{prop:proj}. \; %
\caption{\textsc{Optimal-Run}$(\MM,\phi)$}
\label{alg:complete_alg}
\end{algorithm}

Combining Lemma~\ref{lemma:1}, Theorem~\ref{thm:min_bot_cyc}, and
Proposition~\ref{prop:proj}, we obtain the following result.
\begin{theorem}[Correctness of \textsc{Optimal-Run}]
The \textsc{Optimal-Run} algorithm solves Problem~\ref{problem:1}.
\end{theorem}

\section{Experiments}
\label{sec:experiment}

We have implemented the \textsc{Optimal-Run} algorithm in simulation
and on a physical road network testbed.  The road network shown in
Fig.~\ref{fig:road_network} is a collection of roads, intersections,
and parking lots, connected by a simple set of rules ({\it e.g.,} a
road connects two (not necessarily different) intersections, the
parking lots can only be located on the side of a road). The city is
easily reconfigurable through re-taping.  The robot is a Khepera III
miniature car. The car can sense when entering an intersection from a
road, when entering a road from an intersection, when passing in front
of a parking lot, when it is correctly parked in a parking space, and
when an obstacle is dangerously close.  The car is programmed with
motion and communication primitives allowing it to safely drive on a
road, turn in an intersection, and park.  The car can communicate
through Wi-Fi with a desktop computer, which is used as an interface
to the user ({\it i.e.,} to enter the specification) and to perform
all the computation necessary to generate the control strategy.  Once
computed, this is sent to the car, which executes the task
autonomously by interacting with the environment.

Modeling the motion of the car in the road network using a weighted
transition system (Def. \ref{def:transition system}) is depicted in
Fig.~\ref{fig:platform} and proceeds as follows.  The set of states
$Q$ is the set of labels assigned to the intersections, parking lots,
and branching points between the roads and parking lots. 
The transition relation $R$ shows how the regions are connected and
the transitions' labels give distances between them (measured in
inches).  In our testbed the robot moves at constant speed $\nu$, and
thus the distances and travel times are equivalent.  For these
experiments, the robot can only move on right hand lane of a road and
it cannot make a U-turn at an intersection. To capture this, we model
each intersection as four different states. Note that, in reality,
each state in $Q$ has associated a set of motion primitives, and the
selection of a motion primitive ({\it e.g.,} {\it go\_straight}, {\it
  turn\_right}) determines the transition to one unique next
states. This motivates our assumption that the weighted transition
system from Def. \ref{def:transition system} is deterministic, and
therefore its inputs can be removed.

\begin{figure}
\centering
\scalebox{0.59}{
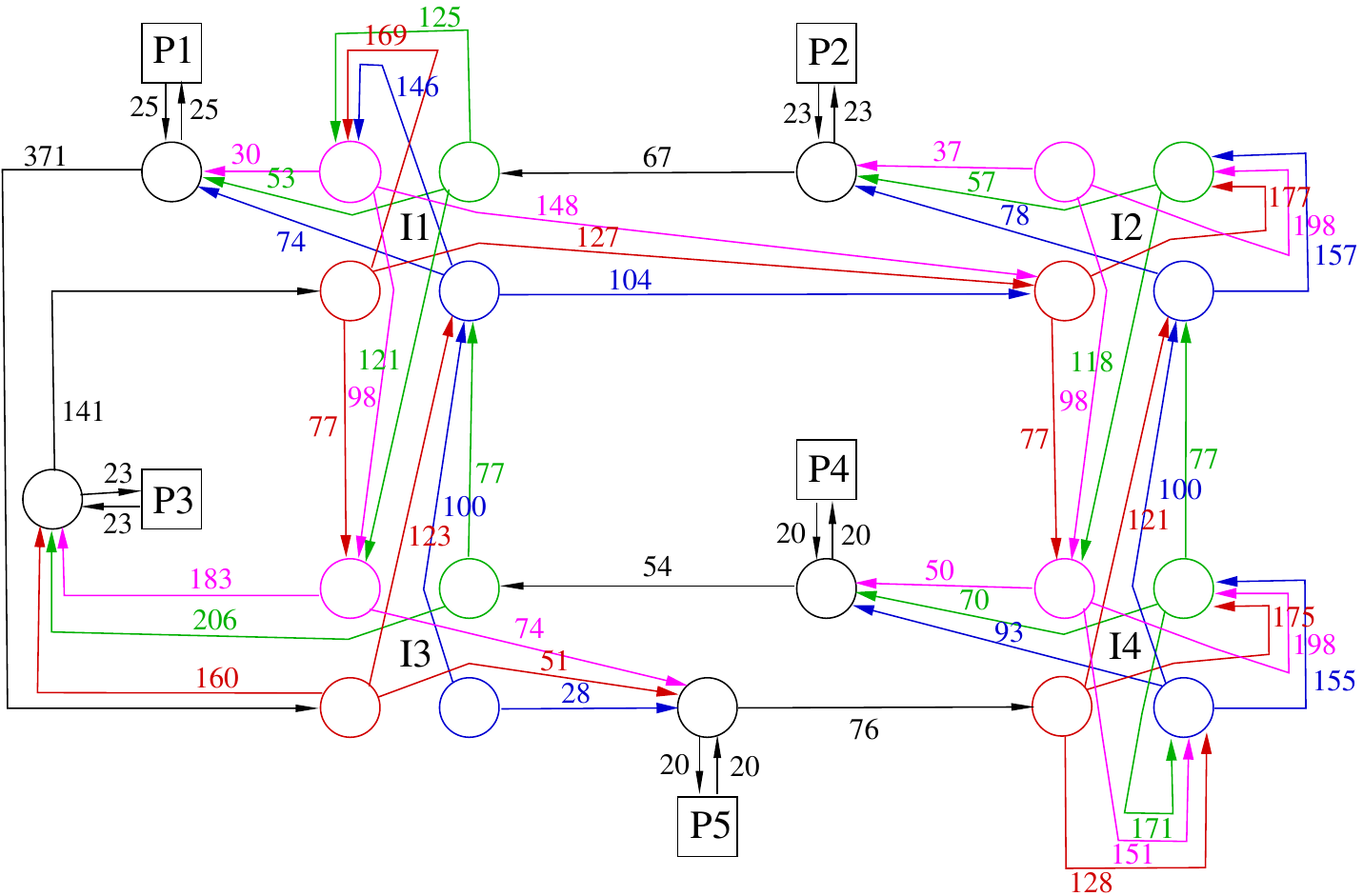
}
\caption{The weighted transition system for the road network in
  Fig.~\ref{fig:road_network}.}
\label{fig:platform}
\end{figure}

In our experiment, we have considered the following task. Parking
spots $\mathrm{P2}$ and $\mathrm{P3}$ in Fig.~\ref{fig:platform} are
data upload locations (light shaded in Fig.~\ref{fig:trajectories})
and parking spots $\mathrm{P1}$, $\mathrm{P4}$, and $\mathrm{P5}$ are
data gather locations (dark shaded in
Fig.~\ref{fig:trajectories}). The optimizing proposition $\pi$ in LTL
formula~\eqref{eq:temp_spec} is $\pi := \mathrm{P2} \vee \mathrm{P3}$,
i.e. we want to minimize the time between data uploads. Both upload
locations provide the same service. On the other hand, data gather
locations are unique and provide the robot with different kind of
data. Assuming infinite runs of the robot in the environment, the
motion requirements can be specified as LTL formulas, where atomic
propositions are simply names of the parking spots. Namely, in the
formula $\varphi$ of the LTL formula~\eqref{eq:temp_spec}, we demand
the conjunction of the following:
\begin{itemize}\itemsep0ex
 \item The robot keeps visiting each data gather location.
$$\Always \, \Event \, \mathrm{P1} \wedge \Always \, \Event \, \mathrm{P4} \wedge \Always \, \Event \, \mathrm{P5}$$
\item Whenever the robot gathers data, it uploads it before doing
  another data gather.
$$\Always \, ((\mathrm{P1} \vee \mathrm{P4} \vee \mathrm{P5})\Rightarrow \Next \, (\neg(\mathrm{P1} \vee \mathrm{P4} \vee \mathrm{P5}) \, \Until \, (\mathrm{P2} \vee \mathrm{P3})))$$
\item Whenever the robot uploads data, it does not visit an upload
  location again before gathering new data.
$$\Always \, ((\mathrm{P2} \vee \mathrm{P3})\Rightarrow \Next \, (\neg (\mathrm{P2} \vee \mathrm{P3}) \, \Until \, (\mathrm{P1} \vee \mathrm{P4} \vee \mathrm{P5})))$$
\end{itemize}
Note that the above specifications implicitly enforce $\Always
\;\Event \; \pi$.  Running the \textsc{Optimal-Run} algorithm, we
obtain the solution as illustrated in the top three environment shots
in Fig.~\ref{fig:trajectories}. The transition system has 26 states,
and the B\"uchi automaton had 16 states, giving a product automaton
with 416 states.  In the product automaton, $F_\PA$ contained 52
states, and $S_{\PA,\opt}$ contained 32 states.  The
\textsc{Optimal-Run} algorithm ran in approximately 6~seconds on a
standard laptop.  The value of the cost function was 9.13 meters, which
corresponded to a robot travel time of 3.6
minutes 
(i.e., the maximum travel time between uploads was 3.6 minutes). Our
video submission displays the robot trajectory for this run and
Fig.~\ref{fig:snapshots} shows two snapshots from the video.
\begin{figure}
  \centering
  \includegraphics[width=0.48\linewidth]{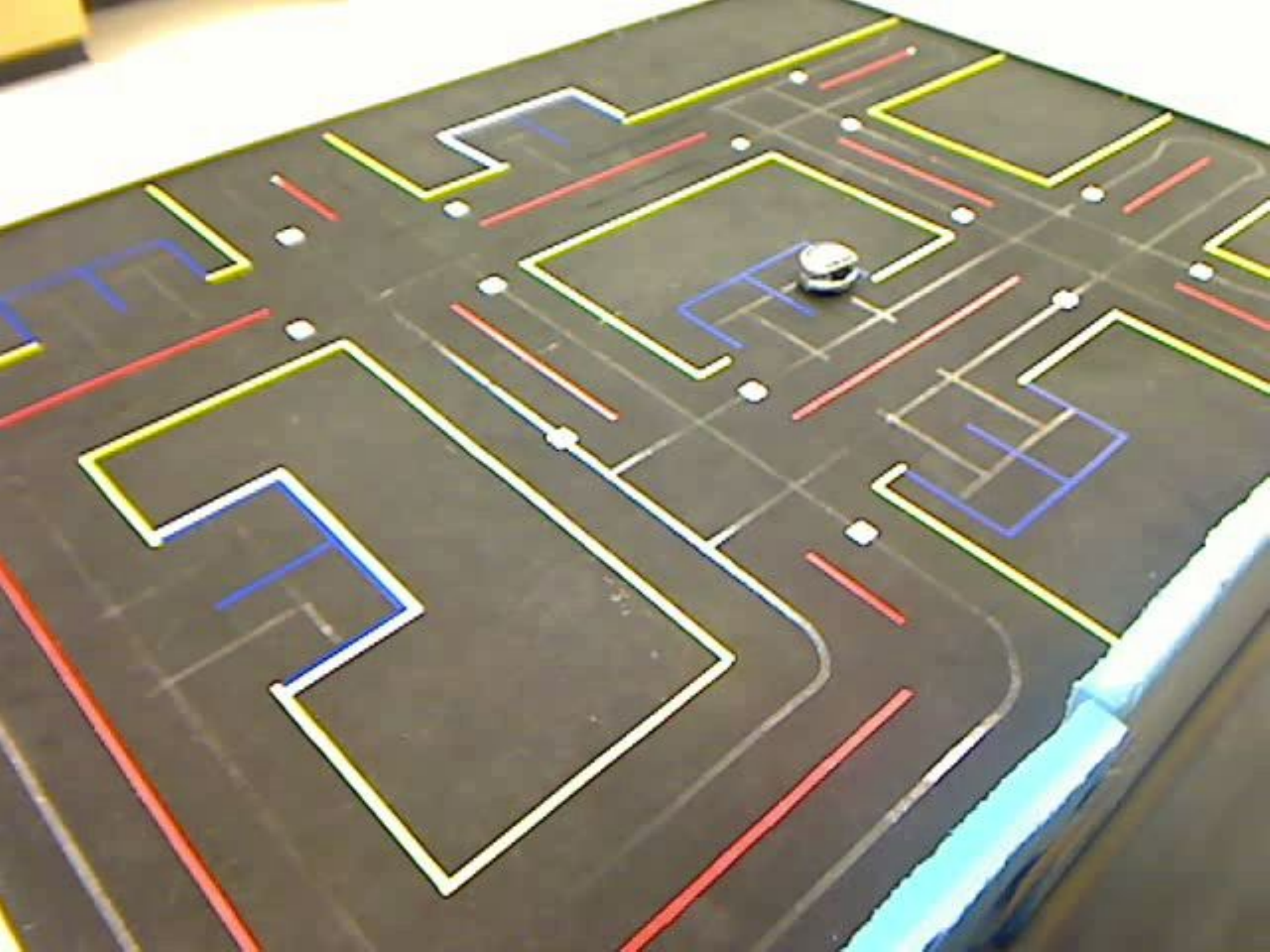}\hfill
    \includegraphics[width=0.48\linewidth]{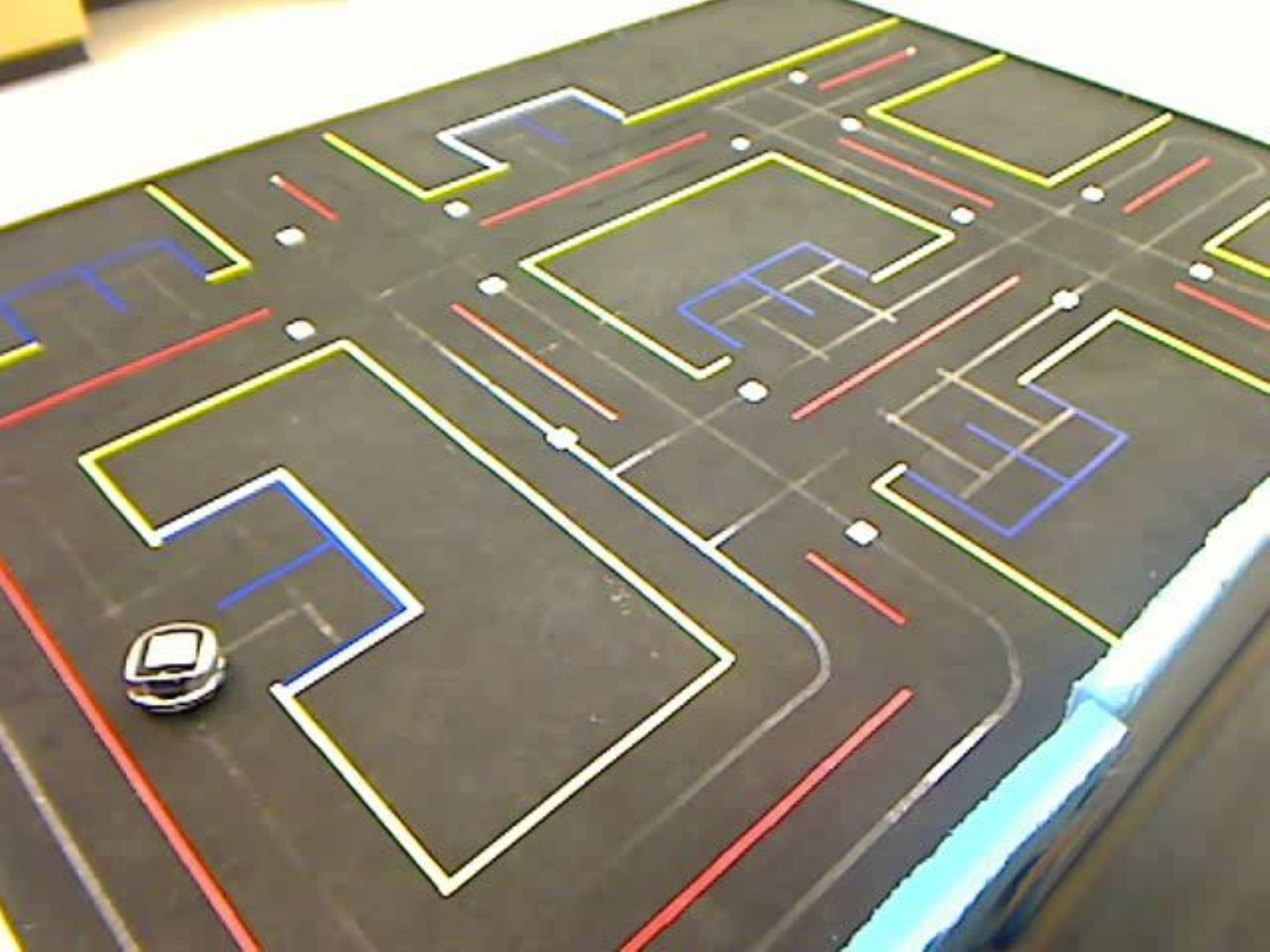}
    \caption{Two snapshots of the robot in road network.  In the left
      figure the robot is gathering data, and in the right figure it
      is about to upload.
    }
  \label{fig:snapshots}
\end{figure}

The bottom three shots in Fig.~\ref{fig:trajectories} illustrate the
situation with the same motion requirements and a further restriction
saying that the robot cannot upload data in $\mathrm{P2}$ after data
is gathered in location $\mathrm{P5}$: $\Always \, (\mathrm{P5}
\Rightarrow (\neg \mathrm{P2} \,\Until\, \mathrm{P3}))$. In this case
the B\"uchi automaton contained 29 states, the algorithm ran in 22
seconds, and the value of the cost function was 9.50 meters with a travel
time of 3.77 minutes.

\begin{figure}
\centering
\includegraphics[width=\linewidth]{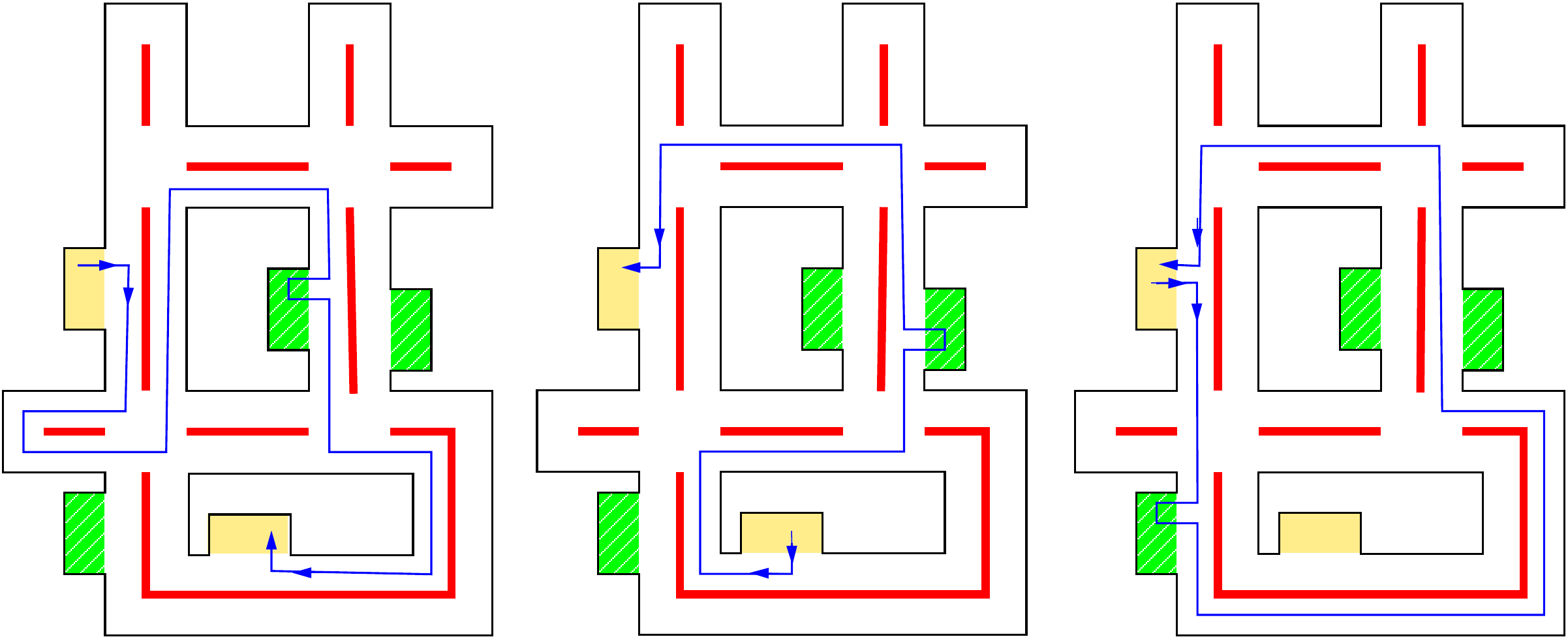} \\[2em]
\includegraphics[width=\linewidth]{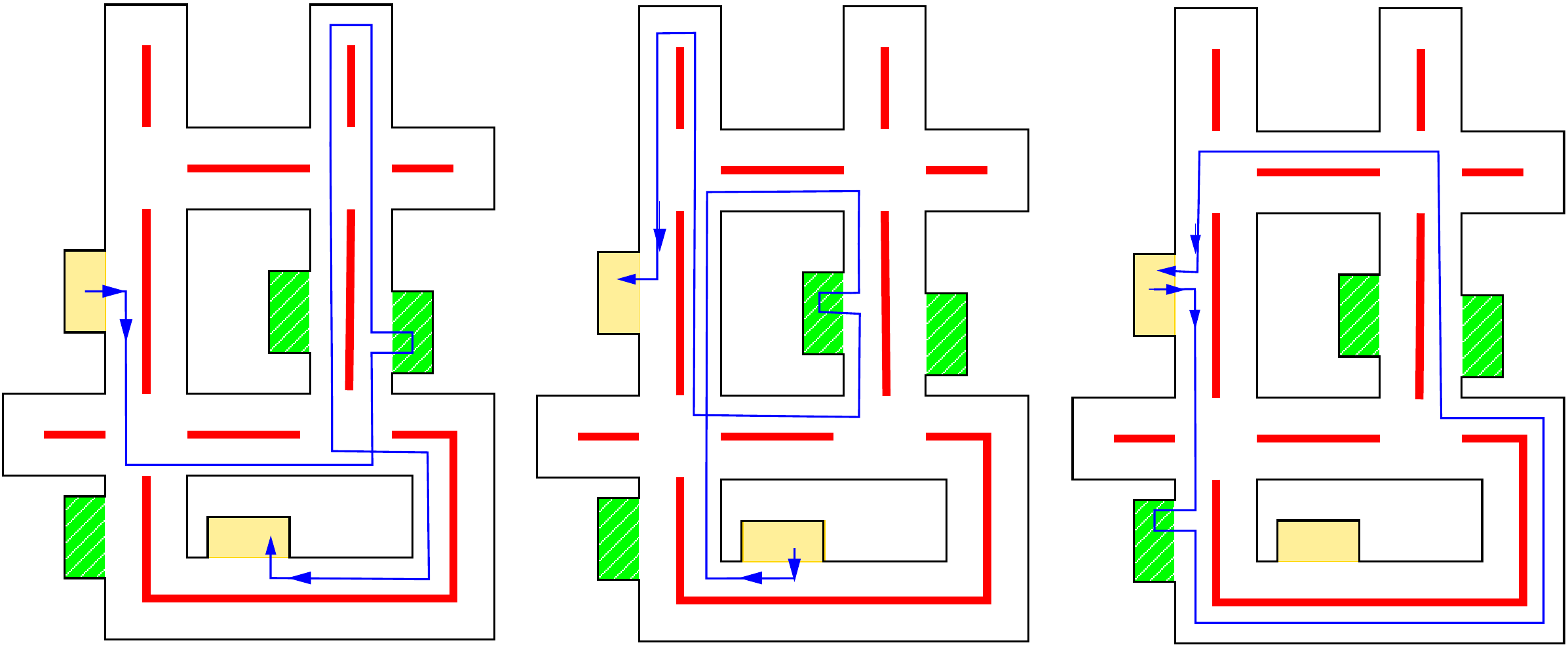} 
\caption{The robot trajectories (blue arrows) for the data gathering
  mission.  Green (dark shaded) areas are data-gathering locations,
  and yellow (light shaded) areas are upload locations.  The bottom
  three figures show the new robot trajectory when we restrict data
  upload to location $\mathrm{P3}$ (the bottom yellow location) after
  each data-gather at $\mathrm{P5}$ (the rightmost green
  location).}
\label{fig:trajectories}
\end{figure}

\section{Conclusions and Future Directions}
\label{sec:conc}

In this paper we presented a method for planning the optimal motion of
a robot subject to temporal logic constraints.  The problem is
important in applications where the robot has to perform a sequence of
operations subject to external constraints.  We considered temporal
logic specifications which contain a single \emph{optimizing
  proposition} that must be repeatedly visited.  We provided a method
for computing a valid robot trajectory that minimizes the maximum time
between satisfying instances of the optimizing proposition.  We
demonstrated our method for a robotic data gathering mission in a city
environment.

There are many promising directions for future work.  We are looking
at ways to extend the cost functions that can be optimized.  In
particular, we are looking at extensions to more general types of
patrolling problems.  Another interesting direction is the extension
to multiple robots.  This naturally leads to developing solutions that
are distributed among the robots.

\section*{Acknowledgements}
We thank Yushan Chen and Samuel Birch at Boston University for their
work on the road network platform.

\end{document}